\documentclass{article}

\usepackage{microtype}
\usepackage{graphicx}
\usepackage{subcaption}
\usepackage{booktabs}
\usepackage{hyperref}
\usepackage{algorithm}
\usepackage{algorithmic}

\usepackage[accepted]{icml2026}

\usepackage{amsmath}
\usepackage{amssymb}
\usepackage{amsthm}
\usepackage{mathtools}
\usepackage{multirow}
\usepackage{pifont}

\usepackage{amsmath,amsfonts,bm}

\def\eqref#1{equation~\ref{#1}}

\def\1{\bm{1}}

\DeclareMathAlphabet{\mathsfit}{\encodingdefault}{\sfdefault}{m}{sl}
\SetMathAlphabet{\mathsfit}{bold}{\encodingdefault}{\sfdefault}{bx}{n}

\DeclareMathOperator*{\argmax}{arg\,max}

\newtheorem{theorem}{Theorem}

\newtheorem{proposition}[theorem]{Proposition}

\icmltitlerunning{Trajectory-Level Speculative Decoding for Diffusion Language Models}

\begin{document}

\twocolumn[
\icmltitle{Trajectory-Level Speculative Decoding \\
for Diffusion Language Models}

\icmlsetsymbol{equal}{*}

\begin{icmlauthorlist}
\icmlauthor{Tianxiang Pan}{equal,comp}
\icmlauthor{Baitao Gong}{equal,comp}
\icmlauthor{Mo Guang}{comp}
\icmlauthor{Hongwei Yong}{comp}
\icmlauthor{Tianpeng Jiang}{comp}
\icmlauthor{Yaqian Li}{comp}
\icmlauthor{Zheng Cao}{comp}
\icmlauthor{Kaiwen Long}{comp}

\end{icmlauthorlist}

\icmlaffiliation{comp}{Li Auto Inc., Shanghai, China}
\icmlcorrespondingauthor{Yaqian Li}{liyaqian@lixiang.com}

\icmlkeywords{Machine Learning, Diffusion Models, Language Models, Speculative Decoding}

\vskip 0.3in
]

\printAffiliationsAndNotice{\icmlEqualContribution}

\begin{abstract}
Diffusion-based language models (dLLMs) enable parallel token generation through iterative denoising, but existing decoding strategies collapse to single-token generation under low confidence, severely limiting throughput. Unlike autoregressive models where speculative decoding operates on token sequences in a fixed left-to-right order, dLLMs require speculating over \emph{denoising trajectories}—sequences of multi-token updates with explicit positions and unmasking orders. We develop a trajectory-level speculative framework that constructs draft denoising trajectories via confidence-stratified tree exploration and verifies them through blockwise parallel evaluation with bidirectional attention masking. Our method further introduces inter-block speculation, exploiting diffusion models' bidirectional structure to perform cross-block lookahead. We formally characterize when this approach is exact and identify trajectory drift as the fundamental cost of increased parallelism. Building on Fast-dLLM's dual-cache infrastructure, our framework reduces denoising iterations by 30-40\% and increases tokens-per-step from 2.6 to 4.3, achieving 7-14$\times$ speedup over vanilla dLLMs and 1.3$\times$ over Fast-dLLM with less than 1\% accuracy change across reasoning and code benchmarks.
\end{abstract}

\section{Introduction}

Speculative decoding has emerged as a powerful technique for accelerating autoregressive language models~\citep{leviathan2023fast,chen2023accelerating}, enabling multiple tokens to be generated per forward pass through draft-and-verify mechanisms. However, these methods fundamentally rely on the causal, left-to-right structure of autoregressive generation and cannot be directly applied to diffusion-based language models (dLLMs)~\citep{nie2025large,ye2025dream}, which generate text through iterative parallel denoising rather than sequential token prediction.

Diffusion language models offer parallelism advantages: unlike autoregressive models that generate tokens sequentially, dLLMs can decode multiple tokens per denoising step. However, existing strategies (threshold-based or top-$k$ selection) work well under high confidence but collapse to single-token generation when confidence is low—\emph{low-confidence degeneration}—severely limiting throughput.

\textbf{Motivating observation.}
Through empirical analysis on Fast-dLLM~\citep{wu2025fast}, we observe that low-confidence degeneration events occur in over 80\% of decoding steps across reasoning and code-generation benchmarks. This severely limits throughput: while the model has the capacity to decode 32 tokens per step (one block), it often decodes only 1-3 tokens. Crucially, we find that even when a position has low confidence, the correct token frequently appears in the top-$k$ candidates ($k \approx 3$). This suggests an opportunity: rather than discarding low-confidence positions, we can speculate over multiple candidate trajectories in parallel and verify them jointly.

\begin{figure}[t]
\centering
\includegraphics[width=\columnwidth]{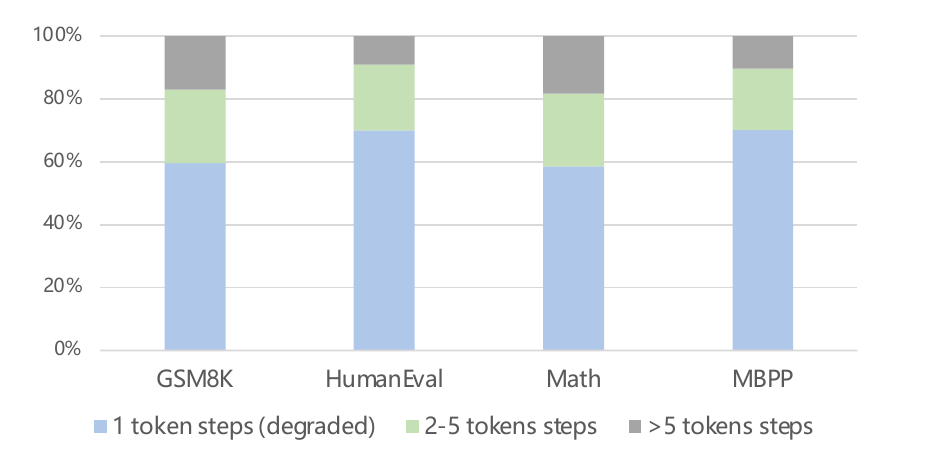}
\caption{Visualization of low-confidence degeneration in parallel decoding for dLLMs. When confidence drops below threshold, existing methods collapse to single-token generation, severely limiting throughput despite dLLMs' inherent parallel capacity.}
\label{fig:problemvis}
\end{figure}

\textbf{Why existing speculative decoding fails for dLLMs.}
Autoregressive speculative decoding~\citep{pmlr-v235-cai24b,li2024eagle} operates in \emph{token space}: it constructs draft token sequences in a fixed left-to-right order and verifies them using causal masking. This design exploits the deterministic structure of causal generation. In contrast, diffusion models generate through \emph{trajectory space}—iterative refinement where tokens are unmasked in arbitrary order via bidirectional attention. A "draft" in this setting must specify not just token identities, but also their positions and unmasking order across multiple denoising steps. This fundamental difference renders token-space speculation insufficient: diffusion models require trajectory-space speculation over sequences of multi-token updates, verified through blockwise attention that isolates competing denoising paths while preserving bidirectional dependencies.

\textbf{Our approach: trajectory-level speculation.}
We develop a framework operating in trajectory space, constructing and verifying candidate \emph{denoising trajectories}—complete sequences of block-level updates with explicit ordering and positions. This requires three components absent in token-space speculation:

(1) Tree-based draft trajectory construction. We construct a tree where each path encodes a complete denoising trajectory (not just token sequences). Nodes represent block-level updates, and paths capture alternative unmasking orders, exploiting the any-order generation capacity of diffusion models.

(2) Blockwise parallel verification. We design a blockwise attention masking scheme that isolates trajectory branches during verification while preserving bidirectional attention to context. This enables parallel evaluation of multiple independent denoising paths in a single forward pass—structurally different from causal verification in AR models.

(3) Inter-block speculation. Exploiting bidirectional visibility unique to diffusion models, we introduce cross-block lookahead: monitoring confidence in block $B_{t+1}$ while decoding $B_t$, and triggering joint speculation when conditions allow. This segmented parallelism has no analog in left-to-right AR decoding.

\textbf{Relationship to infrastructure and baselines.}
Our work builds on Fast-dLLM~\citep{wu2025fast}, which introduced a dual-cache mechanism that reduces per-step computation by reusing key-value activations of prefix and suffix tokens within each block. This provides the systems-level infrastructure on which our algorithmic contribution operates. Fast-dLLM's dual cache addresses per-step efficiency; our trajectory-level speculation addresses the number of steps required. The two optimizations are orthogonal and cumulative: Fast-dLLM reduces FLOP per step, while our method reduces steps overall. We evaluate our approach against Fast-dLLM (Dual Cache) as the baseline, demonstrating a 30\% reduction in end-to-end latency and a 40\% reduction in denoising iterations on top of this already-optimized system.

\textbf{Contributions.}
Our main contributions are:
\begin{itemize}
\item We identify low-confidence degeneration as a fundamental bottleneck in parallel diffusion decoding and formulate it as a trajectory-space speculation problem—showing that correct tokens often appear in top-$k$ candidates even under low confidence.
\item We develop a trajectory-level speculative decoding framework for dLLMs that operates over complete denoising paths rather than token sequences, including: (1) tree-based trajectory construction with confidence-stratified expansion, (2) blockwise parallel verification that isolates trajectory branches while preserving bidirectional context, and (3) inter-block speculation exploiting cross-block lookahead unique to bidirectional models.
\item We establish formal conditions under which trajectory-space speculation is exact (Theorem~\ref{thm:exactness}), characterize trajectory drift as the fundamental cost of increased parallelism, and provide acceptance rate bounds that justify our hybrid tree design.
\item Empirically, our method reduces denoising steps by 30-40\%, increases tokens-per-step from 2.6 to 4.3, and achieves 7-14$\times$ speedup over vanilla dLLMs and 1.3$\times$ over Fast-dLLM, with less than 1\% accuracy change across reasoning and code benchmarks.
\end{itemize}


\section{Background and Related Work}

\subsection{Masked Diffusion Language Models}

Masked Diffusion Language Models (MDLMs)~\citep{sahoo2024simple,zheng2023reparameterized} formulate text generation as an iterative denoising process. Given a target sequence $\mathbf{x}_0 = (x_1, \ldots, x_n)$, the forward diffusion gradually masks tokens:
\begin{equation}
\begin{split}
q(\mathbf{x}_t \mid \mathbf{x}_{t-1}) = \prod_{i=1}^n \big[(1-\beta_t)\,\delta(x_{t,i} = x_{t-1,i}) \\
+ \beta_t\,\delta(x_{t,i} = [\text{MASK}])\big],
\end{split}
\end{equation}
where $\beta_t$ is the corruption rate. The reverse process $p_\theta(\mathbf{x}_{t-1} \mid \mathbf{x}_t)$ progressively denoises masked positions back to concrete tokens. Representative models include LLaDA~\citep{nie2025large}, Dream~\citep{ye2025dream}, and Seed~\citep{song2025seed}.

\textbf{Parallel decoding strategies.}
Two primary approaches exist for parallel decoding in dLLMs:
(1) Top-$k$ decoding~\citep{nie2025large} selects the $k$ most confident tokens at each step, guaranteeing fixed parallelism but risking accuracy when low-confidence tokens are included.
(2) Threshold-based decoding~\citep{wu2025fast} accepts all tokens exceeding confidence threshold $\tau$, adapting parallelism to model certainty but collapsing to top-1 when no token passes the threshold.

Both strategies exhibit low-confidence degeneration: when faced with uncertain positions, they revert to single-token generation, severely limiting throughput.

\subsection{Speculative Decoding in Autoregressive Models}

Speculative decoding~\citep{leviathan2023fast,chen2023accelerating} accelerates autoregressive LLMs by drafting candidate continuations and verifying them in parallel. Representative methods include:

Multi-model approaches~\citep{xia2022speculative,miao2024specinfer} use a lightweight draft model to propose candidates, which a target model then verifies. Self-speculative methods~\citep{pmlr-v235-cai24b,li2024eagle} derive drafts from the target model itself using additional prediction heads or cached representations.

\textbf{Why autoregressive speculation does not transfer to dLLMs.}
Autoregressive speculative decoding relies on three properties absent in diffusion models:
\begin{enumerate}
\item Fixed generation order: Tokens are always generated left-to-right, so draft construction is deterministic given the prompt.
\item Causal verification: A single forward pass with causal masking suffices to verify all draft positions in parallel.
\item Exact acceptance: The verification procedure is mathematically lossless—accepted tokens match exactly what the target model would have produced.
\end{enumerate}

Diffusion models violate all three: they use bidirectional attention, generate tokens in arbitrary order, and support multiple valid denoising trajectories. A "draft" must therefore specify not just token identities but also their positions and unmasking order—a complete trajectory. Verification requires blockwise masking to isolate competing branches, and exactness cannot be guaranteed when altering the parallelism schedule.

\subsection{Speculative Approaches for Diffusion Models}

Recent work has begun exploring speculative techniques for diffusion models. \citet{de2025accelerated,hu2025diffusion} study speculative sampling for continuous denoising diffusion (DDPM), focusing on image generation. For discrete diffusion language models, several concurrent efforts have emerged:

\textbf{Self-speculative decoding}~\citep{gao2025self} introduces token-level speculation for diffusion LMs, where the model drafts individual token predictions that are then verified. This approach operates at token granularity and focuses on single-step speculation.

\begin{table}[t]
\centering
\caption{Comparison of speculative decoding methods for dLLMs.}
\label{tab:method_comparison}
\small
\setlength{\tabcolsep}{3pt}
\begin{tabular}{lccc}
\toprule
Property & Self-Spec & Spiffy & Ours \\
\midrule
Speculation granularity & Token & Block & Trajectory \\
Cross-block lookahead & \ding{55} & \ding{55} & \ding{51} \\
Tree-based exploration & \ding{55} & \ding{55} & \ding{51} \\
Lossless verification & \ding{51} & \ding{51} & $\sim$\ding{51} \\
Calibration overhead & Low & High & Low \\
Speedup (7B models) & $\sim$3.46$\times$ & $\sim$7$\times$ & $\sim$14$\times$ \\
\bottomrule
\end{tabular}
\end{table}

\textbf{Spiffy}~\citep{agrawal2025spiffy} proposes block-level speculative decoding for diffusion language models using graph-based calibration to maintain lossless verification. Their method performs parallel exploration over multiple draft blocks within a decoding block (intra-block speculation) and focuses on preserving exactness under speculative execution. In contrast, our method formulates speculation at the trajectory level, jointly modeling denoising trajectories that can span consecutive blocks through inter-block speculation. 

\textbf{Our approach} differs in problem formulation and method design: (1) we identify and address low-confidence degeneration—showing that parallelism collapses not due to incorrectness but due to fixed decoding policies—whereas concurrent methods focus on either token-level drafting or calibration for exactness, (2) we explore multiple trajectory alternatives via tree-structured drafts to sustain parallelism under uncertainty, enabling broader speculation than single-block or token-level approaches, (3) we introduce inter-block speculation that exploits bidirectional attention to perform cross-block lookahead, which single-block designs cannot express, and (4) our tree-based verification is simpler than graph-based calibration and achieves higher acceptance rates (3.8-4.3 tokens/step vs. Spiffy's reported rates).

Table~\ref{tab:method_comparison} contrasts key properties across the three approaches. Self-speculative decoding operates at token granularity, while both Spiffy and our method perform block-level speculative exploration for diffusion LMs. The key distinction is that Spiffy focuses on intra-block speculation with calibrated exact verification, whereas our method additionally introduces inter-block trajectory speculation across consecutive decoding blocks. We further report end-to-end TPS, latency, denoising steps, and accuracy together, while Spiffy primarily reports NFE-based speedup (see Appendix~\ref{app:concurrent} for detailed comparison). We note that other concurrent approaches may exist, and the landscape of diffusion LM acceleration continues to evolve.

\textbf{Complementarity with other acceleration techniques.}
Our trajectory-level speculation operates at the algorithmic level of the decoding process and is conceptually orthogonal to system-level optimizations such as KV cache management~\citep{wu2025fast}, architectural improvements (sparse attention, quantization), and hardware-specific optimizations. These methods target different components of the inference stack: we reduce the number of denoising iterations, while system optimizations reduce per-iteration cost. Combining these approaches is conceptually feasible and could yield cumulative benefits, though such integration is beyond the scope of this work.

\section{Method}

\subsection{Overview}

Our framework operates in three stages:
(1) Candidate collection: Gather top-$k$ tokens at low-confidence positions from the previous denoising step.
(2) Tree-based draft trajectory generation: Incrementally construct a tree of candidate trajectories, where each node represents a block of tokens and each path represents a possible denoising sequence.
(3) Parallel verification and acceptance: Evaluate all draft blocks jointly using blockwise attention masking, then select the longest valid path via greedy or longest-path verification.

These stages reuse the model's forward pass: the verification of step $t$ simultaneously provides logits for drafting step $t+1$, minimizing overhead.

\begin{figure}[t]
\centering
\includegraphics[width=\columnwidth]{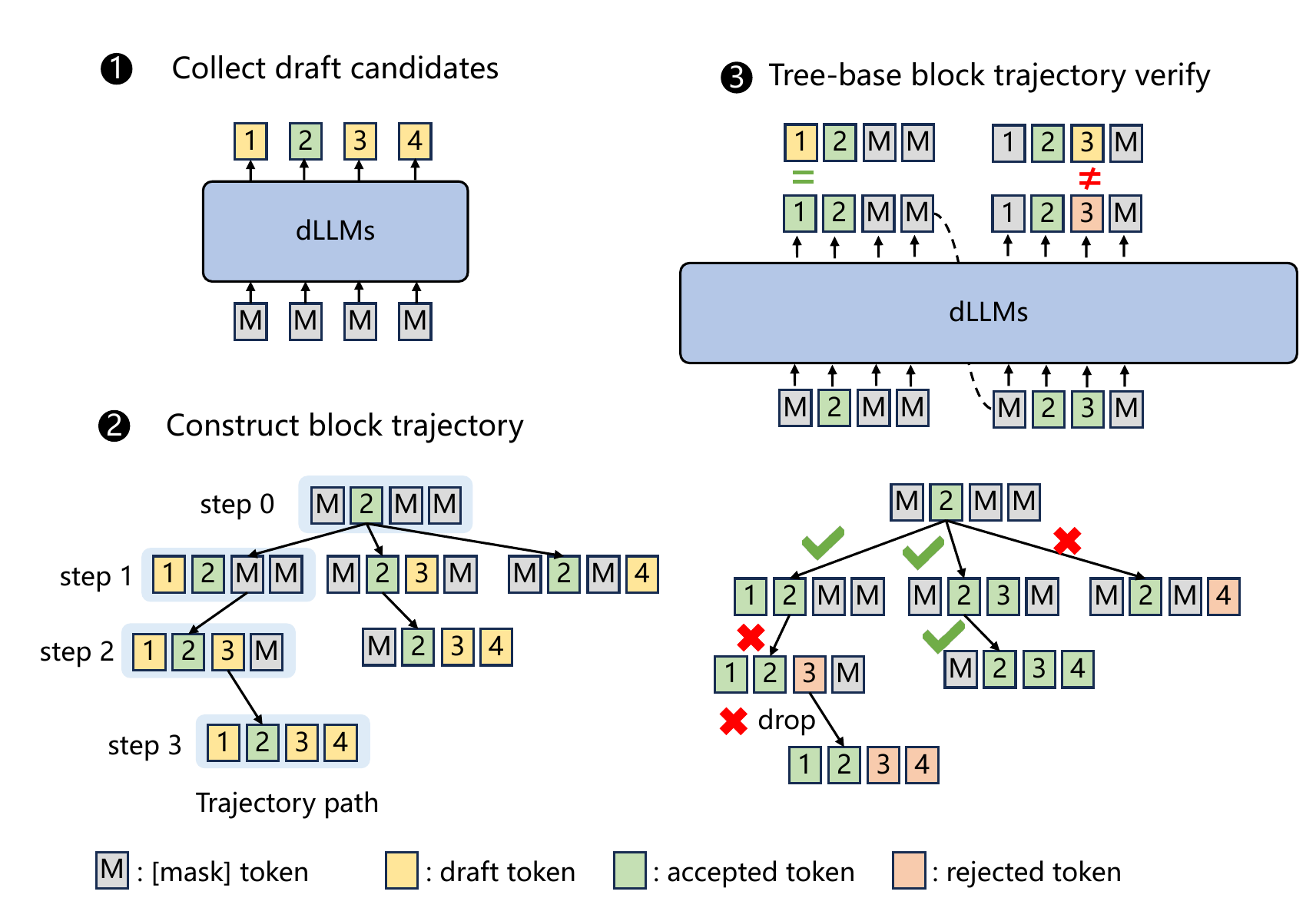}
\caption{Overview of trajectory-level speculative decoding framework. The method constructs a tree of candidate denoising trajectories, verifies them in parallel using blockwise attention masking, and selects the longest valid path.}
\label{fig:generalmodel}
\end{figure}

\subsection{Tree-Based Draft Trajectory Generation}

\textbf{Motivation: why low-confidence tokens matter.}
Existing parallel decoding strategies discard tokens below threshold $\tau$, treating low confidence as synonymous with incorrectness. This conflates two distinct phenomena: tokens with low confidence due to model uncertainty, and tokens excluded because they fall outside the parallelism budget. Our key observation: low confidence does not imply incorrectness—it often reflects ambiguity among several plausible alternatives~\citep{kim2025train}. Empirically, we find that even at low-confidence positions, the correct token frequently appears within the top-$k$ candidates ($k \approx 3$). This motivates a principled approach: rather than abandoning parallelism at uncertain positions, we explore multiple candidate trajectories in parallel and verify them jointly, thereby sustaining throughput without compromising correctness guarantees.

\textbf{Candidate collection.}
At each denoising step, we define token confidence as:
\begin{equation}
c(x) = \max(\text{softmax}(\text{logits}_t(x))).
\end{equation}
We collect top-$k$ candidates at positions where $c(x) < \tau$—i.e., tokens that are low-confidence but potentially correct. This forms the pool for trajectory exploration.

\textbf{Tree construction.}
We organize candidates into a tree where each node corresponds to a block of tokens (e.g., 32 tokens). The root node contains all high-confidence tokens (those exceeding $\tau$). Children represent alternative continuations constructed from the candidate pool. To balance coverage and computational cost, we adopt a \emph{hybrid expansion strategy}:
\begin{itemize}
\item Layer 1 (root): Expand with top-$k$ candidates to maximize diversity.
\item Deeper layers: Expand each node with only its top-1 continuation, yielding a compact tree that concentrates exploration along high-probability paths.
\end{itemize}

This design keeps the tree compact ($O(D \cdot W)$ nodes for depth $D$ and width $W$) while preserving high-likelihood trajectories. For example, a W2D2(3) tree has width 2, depth 2, and 3 total children blocks. Ablation studies (Section~\ref{sec:ablation}) show this hybrid strategy achieves the best acceptance-rate-to-overhead trade-off.

\textbf{Algorithm.}
Algorithm~\ref{alg:tree} implements \emph{confidence-stratified expansion}: high-confidence tokens form the root block, low-confidence positions trigger branching with adaptive depth control, balancing exploration and overhead.

\begin{algorithm}[t]
\caption{Confidence-Stratified Trajectory Tree Construction}
\label{alg:tree}
\begin{algorithmic}[1]
\STATE \textbf{Input:} Token logits $\mathcal{L}_t$ at step $t$, confidence threshold $\tau$, branching factor $k$, tree depth $D$, tree width $W$
\STATE \textbf{Output:} Draft trajectory tree $\mathcal{T}$ with root and leaf nodes
\vspace{0.1cm}
\STATE \textcolor{blue}{// Phase 1: Confidence stratification}
\STATE $\mathcal{H} \gets \{i : \max_v \text{softmax}(\mathcal{L}_t[i, v]) > \tau\}$ \hfill $\triangleright$ High-confidence positions
\STATE $\mathcal{L} \gets \{i : \max_v \text{softmax}(\mathcal{L}_t[i, v]) \leq \tau\}$ \hfill $\triangleright$ Low-confidence positions
\STATE Initialize root node $r \gets (\mathcal{H}, \{\arg\max_v \mathcal{L}_t[i, v]\}_{i \in \mathcal{H}})$
\vspace{0.1cm}
\STATE \textcolor{blue}{// Phase 2: Speculative expansion}
\STATE $\mathcal{T} \gets \{r\}$; $\text{frontier} \gets \{r\}$
\FOR{$\ell = 1$ to $D$}
    \STATE $\text{next\_frontier} \gets \emptyset$
    \FOR{each node $n \in \text{frontier}$}
        \IF{$\ell = 1$}
            \STATE $\text{children} \gets \text{GenerateTopK}(n, \mathcal{L}, k)$ \hfill $\triangleright$ Explore top-$k$ at root
        \ELSE
            \STATE $\text{children} \gets \text{GenerateTop1}(n, \mathcal{L})$ \hfill $\triangleright$ Greedy extension deeper
        \ENDIF
        \STATE $\text{next\_frontier} \gets \text{next\_frontier} \cup \text{children}$
    \ENDFOR
    \STATE $\text{frontier} \gets \text{PruneByConfidence}(\text{next\_frontier}, W)$ \hfill $\triangleright$ Keep top-$W$ by confidence
    \STATE $\mathcal{T} \gets \mathcal{T} \cup \text{frontier}$
\ENDFOR
\STATE \textbf{return} $\mathcal{T}$
\end{algorithmic}
\end{algorithm}

\textbf{Complexity analysis.}
\label{subsec:complexity}
A fully expanded trajectory tree would incur exponential cost in the number of branches.
Our method is specifically designed to avoid this regime through the constrained tree structure described above.

We adopt a \textit{hybrid expansion strategy}:
\begin{itemize}
    \item Only the root node branches with top-$k$ candidates, while all deeper levels follow top-1 expansion.
    \item Top-$k$ expansion at the root layer maximizes diversity where it matters most, while guaranteeing $O(W \cdot D)$ complexity instead of exponential scaling, where $W$ is the tree width and $D$ is the tree depth.
    \item The resulting candidates can then be efficiently processed via parallel blockwise verification.
\end{itemize}

Concretely, with width $W=3$ and depth $D=3$, this yields at most $W \times D = 9$ draft sequences per step, compared to $W^D = 27$ for a fully expanded tree—a $3\times$ reduction in the number of branches due to top-1 expansion in deeper levels.
Tree construction itself requires only $W \times D$ tensor copies of the current block (each of length $L$, the block size), amounting to $O(W \cdot D \cdot L)$ memory operations, followed by a single concatenated forward pass over $(W \cdot D + 1) \cdot L$ tokens.
This represents a constant-factor overhead relative to a standard single-branch step.
Moreover, relative to the total context length and generation length during inference, especially in long-sequence generation scenarios, the additional $W \cdot D \cdot L$ tokens introduced by tree construction constitute a lightweight overhead.

This design is not merely a heuristic simplification, but is motivated by our key observation that uncertainty in dLLMs is highly concentrated at a small number of high-entropy positions, while later steps typically exhibit much higher confidence.
Therefore, allocating branching capacity only at these critical positions is both efficient and sufficient.
Empirically, this constrained tree achieves a strong balance between acceptance rate and computational cost, as shown in the ablation studies in Section~\ref{sec:ablation}.
Given the increased overall latency and only marginal accuracy gains, we practically adopt the \texttt{W3D3(6)} variant rather than the full-tree variant \texttt{W3D3(9)}.

\subsection{Blockwise Parallel Verification}

\textbf{Attention masking for trajectory isolation.}
Unlike autoregressive models that use causal masking, we require a blockwise attention scheme that: (i) preserves bidirectional attention between draft blocks and cached prefix/suffix tokens, and (ii) isolates draft branches to prevent cross-interference.

Specifically, let $P$ and $S$ denote prefix (already-decoded) and suffix (remaining [\textsc{mask}]) tokens. For each draft block $B_i$ in the tree, we construct a mask where:
\begin{itemize}
\item All tokens in $B_i$ attend to all tokens in $P$ and $S$.
\item Tokens in $B_i$ do not attend to tokens in other draft blocks $B_j$ ($j \neq i$).
\end{itemize}

This design maintains the bidirectional structure required by diffusion models while ensuring independent evaluation of each trajectory branch. Figure~\ref{fig:attention_mask} illustrates the blockwise attention mask construction for a tree with multiple draft branches.

\begin{figure}[t]
\centering
\includegraphics[width=\columnwidth]{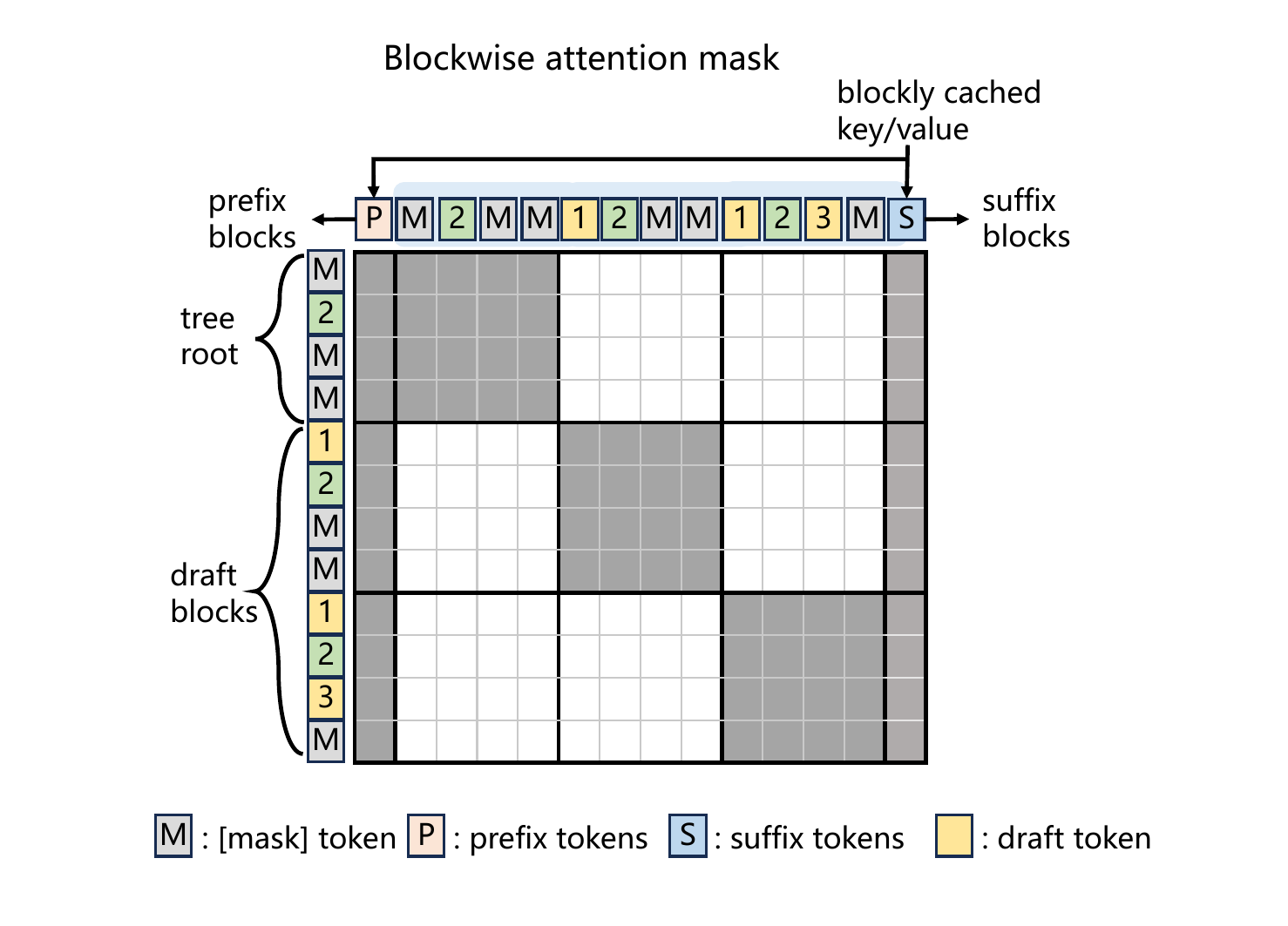}
\caption{Blockwise attention mask construction. Each draft block (shown in different colors) attends bidirectionally to cached prefix and suffix tokens (gray), but draft blocks remain isolated from each other to enable parallel verification of independent trajectories.}
\label{fig:attention_mask}
\end{figure}

\textbf{Verification procedure.}
After concatenating all draft blocks with the input sequence, we perform a single forward pass with the blockwise attention mask. For each block, we check whether its tokens match the model's top predictions (or fall within top-$k$ / exceed $\tau$). We traverse the tree path-by-path, accepting the first valid continuation at each level (greedy search) or the longest valid path overall (longest-path search). In practice, both strategies yield similar performance; we default to greedy search for lower overhead.

\subsection{Inter-Block Speculation}

\textbf{Motivation.}
The methods above accelerate decoding within a single block (intra-block speculation). However, diffusion models' bidirectional attention enables a further optimization: while decoding block $B_t$, we can simultaneously monitor block $B_{t+1}$. If early tokens in $B_{t+1}$ become confident before $B_t$ completes, we can initiate speculation for both blocks jointly.

\textbf{Trigger condition.}
Let $c_t^{(1)}$ and $c_{t+1}^{(1)}$ denote the top-1 confidence in blocks $B_t$ and $B_{t+1}$. We trigger inter-block speculation when:
\begin{equation}
c_{t+1}^{(1)} > c_t^{(1)} \quad \text{or} \quad c_{t+1}^{(1)} > \tau.
\end{equation}

\textbf{Joint tree construction.}
Upon triggering, we build a compact draft tree for $B_{t+1}$ (e.g., W1D3 with only 1-2 children blocks) and verify both trees jointly. The attention mask extends to preserve bidirectional visibility between the roots of $B_t$ and $B_{t+1}$—i.e., draft tokens in $B_{t+1}$ attend to the root of $B_t$, and vice versa. This leverages Fast-dLLM's observation~\citep{wu2025fast} that prefix/suffix KV states remain stable across adjacent denoising steps. Figure~\ref{fig:interblock} illustrates the inter-block speculation mechanism, showing how confidence in the next block can trigger joint speculative decoding across block boundaries.

\begin{figure}[t]
\centering
\includegraphics[width=\columnwidth]{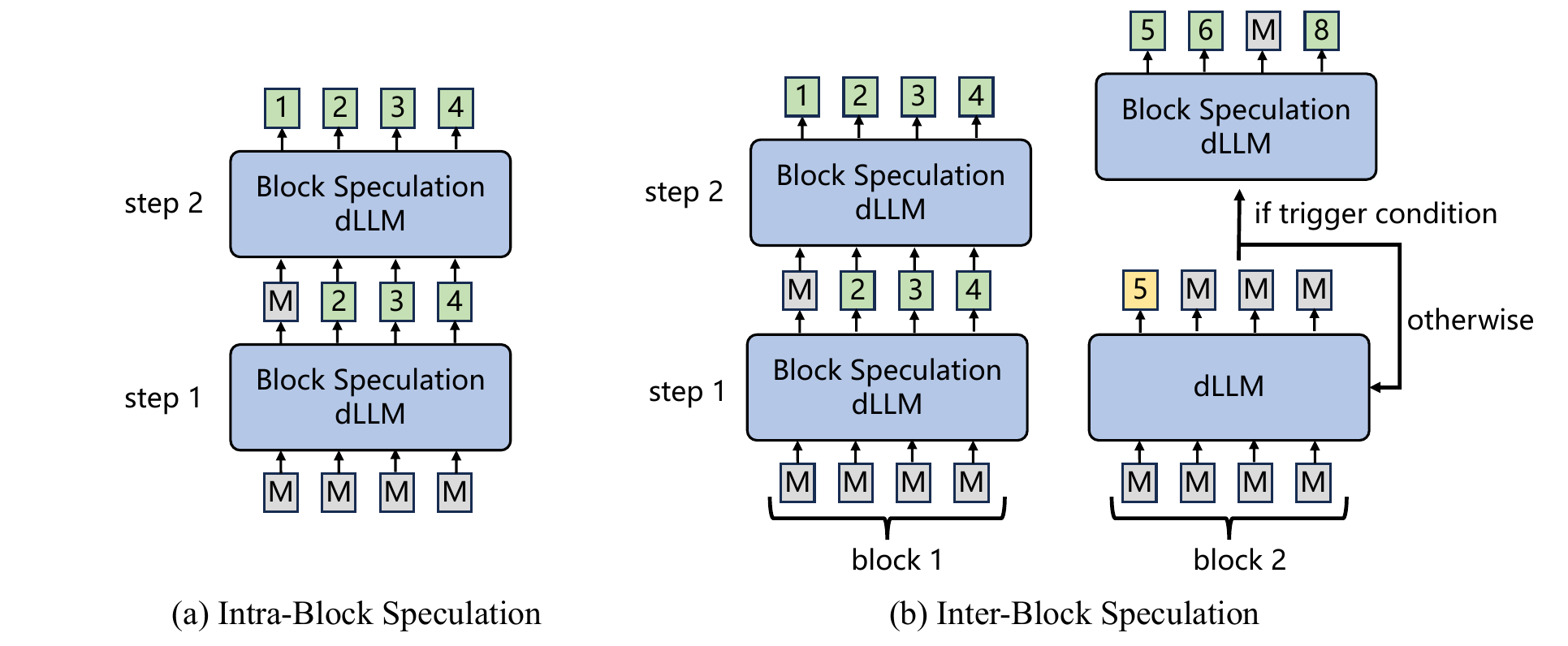}
\caption{Inter-block speculation mechanism. When block $B_{t+1}$ exhibits high early confidence, we construct compact draft trees for both $B_t$ and $B_{t+1}$ and verify them jointly, exploiting the bidirectional attention structure to enable cross-block lookahead.}
\label{fig:interblock}
\end{figure}

Inter-block speculation applies only to pairs of consecutive blocks; extending further would incur excessive verification overhead with diminishing returns. Detailed attention masks and ablation studies are provided in Appendix~\ref{app:interblock}.

\section{Theoretical Analysis}
\label{sec:theory}

Unlike autoregressive speculative decoding, which is mathematically exact, our trajectory-level framework introduces potential trajectory drift due to altered parallelism schedules. We characterize exactness conditions, provide bounds on acceptance rates, and identify when deviations may occur.

\subsection{Exactness Conditions}

\begin{theorem}[Exactness under idealized conditions]
\label{thm:exactness}
Under cached prefix-suffix KV states, no inter-block speculation, and deterministic top-1 decoding, our method produces a denoising trajectory $(i_1, y_1), \ldots, (i_N, y_N)$ identical to Fast-dLLM (Dual Cache), where $i_n$ denotes the position and $y_n$ the token selected at step $n$.
\end{theorem}

\begin{proof}
We prove by strong induction on the number of denoising steps.

\noindent\emph{Base case} ($n=0$): Before the first denoising step, both decoders have identical initial states: all positions masked, no tokens decoded. Thus the trajectories match trivially.

\noindent\emph{Inductive hypothesis}: Assume that for all steps $k < n$, both decoders have produced identical trajectories: $(i_1, y_1), \ldots, (i_{n-1}, y_{n-1})$.

\noindent\emph{Inductive step}: We show that at step $n$, both decoders select the same position-token pair $(i_n, y_n)$.

Let $\mathcal{M}_n$ denote the set of masked positions at step $n$, and let $\mathcal{U}_n = \{(i_1, y_1), \ldots, (i_{n-1}, y_{n-1})\}$ be the set of unmasked tokens. By the inductive hypothesis, both decoders have identical $\mathcal{M}_n$ and $\mathcal{U}_n$.

\noindent\emph{Key observation}: Under cached KV states and no inter-block speculation, both decoders evaluate identical conditional distributions at step $n$. Specifically, for each position $i \in \mathcal{M}_n$, the probability distribution over tokens is:
\begin{equation}
p(x_i = t \mid \mathcal{U}_n, C) = \text{softmax}(f_\theta(i \mid \mathcal{U}_n, C)),
\end{equation}
where $C$ denotes the cached prefix and suffix KV states, and $f_\theta$ is the model's logit function.

Since both decoders use the same model $\theta$, the same cached context $C$, and have identical unmasked tokens $\mathcal{U}_n$, they compute identical logits:
\begin{equation}
f_\theta^{\text{ours}}(i \mid \mathcal{U}_n, C) = f_\theta^{\text{Fast-dLLM}}(i \mid \mathcal{U}_n, C) \quad \forall i \in \mathcal{M}_n.
\end{equation}

Under deterministic top-1 decoding, position $i_n$ is selected as:
\begin{equation}
i_n = \argmax_{i \in \mathcal{M}_n} \max_t p(x_i = t \mid \mathcal{U}_n, C),
\end{equation}
and token $y_n$ is selected as:
\begin{equation}
y_n = \argmax_t p(x_{i_n} = t \mid \mathcal{U}_n, C).
\end{equation}

Since both decoders compute identical distributions, they must select identical $(i_n, y_n)$. This completes the inductive step.

\noindent\emph{Conclusion}: By strong induction, the trajectories match at all steps, hence the full denoising sequences are identical.
\end{proof}

\textbf{Implications.}
Theorem~\ref{thm:exactness} establishes that trajectory-space speculation can be made exact under controlled conditions, formally characterizing when our method preserves the baseline's output distribution. This provides a principled foundation: all deviations stem exclusively from increased parallelism (trajectory drift), not from the speculation mechanism itself. The theorem identifies the exactness boundary and clarifies that any accuracy-speed trade-offs arise from \emph{choosing} higher parallelism, not from algorithmic approximation.

This formalization has practical consequences for algorithm design: to maximize speedup while maintaining exactness guarantees, we must (1) preserve cached KV states across speculative branches to ensure logit consistency, (2) disable inter-block speculation when exactness is required, and (3) verify that tree construction explores only positions and tokens that would be reachable under the baseline's conditional distributions. Our implementation respects these constraints, enabling deployment modes ranging from exact (matching baseline) to high-throughput (controlled drift).

Empirically, we verify this by running both decoders in top-1 mode on HumanEval and GSM8K; outputs and logits match exactly (see Appendix~\ref{app:exactness}, Table~A3), confirming that our implementation correctly preserves the idealized behavior characterized by the theorem.

\subsection{Acceptance Rate Analysis}

The effectiveness of our approach depends on the acceptance rate of draft trajectories. We provide a probabilistic characterization:

\begin{proposition}[Acceptance rate lower bound]
\label{prop:acceptance}
Let $\mathcal{T}$ be a draft tree with depth $D$ and let $\alpha_\ell$ denote the probability that a draft node at level $\ell$ matches the verifier's top-1 prediction. The expected number of accepted tokens is lower-bounded by:
\begin{equation}
\mathbb{E}[\text{accepted}] \geq \sum_{\ell=1}^D \prod_{j=1}^\ell \alpha_j \cdot b_\ell,
\end{equation}
where $b_\ell$ is the block size at level $\ell$.
\end{proposition}

\begin{proof}[Proof sketch]
Consider a path from root to level $\ell$. The probability that all nodes along this path are accepted is $\prod_{j=1}^\ell \alpha_j$ (assuming independence). If accepted, this path contributes $\sum_{j=1}^\ell b_j \geq b_\ell$ tokens. Summing over all paths and applying linearity of expectation yields the lower bound.
\end{proof}

\textbf{Design implications.}
This bound shows that acceptance rates compound multiplicatively across tree depth. This motivates our hybrid expansion strategy: using top-$k$ at the root maximizes $\alpha_1$ (the most important factor), while using top-1 deeper keeps the tree compact without severely degrading $\alpha_\ell$ for $\ell > 1$.

\textbf{Empirical validation.}
We measure acceptance probabilities on HumanEval: $\alpha_1 \approx 0.85$ for high-confidence root blocks and $\alpha_2 \approx 0.68$ for secondary branches. With block size $b=32$, this yields an expected acceptance of $\mathbb{E}[\text{accepted}] \geq 0.85 \times 32 + 0.85 \times 0.68 \times 32 \approx 45.6$ tokens per tree, explaining our observed 3.8 tokens-per-step in Table~\ref{tab:ablation} (accounting for verification overhead and rejected branches).

\section{Experiments}

\subsection{Experimental Setup}

\textbf{Models and baselines.}
We evaluate on LLaDA-Instruct-7B~\citep{nie2025large} and Dream-Instruct-7B~\citep{ye2025dream}. Our baseline is Fast-dLLM (Dual Cache)~\citep{wu2025fast}, which caches prefix and suffix KV states within each block to reduce per-step computation. We compare against vanilla dLLMs, Fast-dLLM (single cache), and concurrent work (Spiffy~\citep{agrawal2025spiffy}, DPad~\citep{chen2025dpad}).

\textbf{Benchmarks.}
We use mathematical reasoning (GSM8K, MATH) and code generation (HumanEval, MBPP) tasks. Generation length is fixed to 512 tokens with block size 32 to ensure fair comparison. Unless otherwise noted, we use the W2D2(3) tree structure on A800 GPUs, as it provides the best latency-efficiency trade-off (Section~\ref{sec:overhead}). This configuration represents typical deployment settings for current discrete diffusion LLMs; we analyze longer generation lengths (up to 2048 tokens) in Appendix A.9, finding consistent speedup patterns.

\textbf{Metrics.}
We report: (1) TPS (tokens per second) and latency, (2) tokens-per-step, (3) denoising steps, and (4) task accuracy. All experiments use a single A800 GPU unless specified.

\subsection{Main Results}

Tables~\ref{tab:llada_results} and~\ref{tab:dream_results} summarize performance on LLaDA and Dream across four benchmarks. Our method consistently reduces denoising steps by 30-43\% relative to Fast-dLLM (Dual Cache) and increases tokens-per-step from 2.6 to 3.8-4.3. This translates to 1.2-1.4$\times$ additional speedup on top of Fast-dLLM's already-optimized baseline, achieving 7-14$\times$ end-to-end speedup over vanilla dLLMs. Accuracy remains within 1\% of the baseline, confirming that trajectory drift is minimal.

\textbf{Task-specific analysis.}
Performance varies by task characteristics. For \textbf{GSM8K} (mathematical reasoning with structured outputs), our method achieves the highest speedup (13.6$\times$ on LLaDA) due to predictable token patterns that enable reliable tree acceptance. For \textbf{HumanEval} (code generation), moderate speedup (6.8$\times$) reflects the challenge of speculating syntax-heavy sequences where single-token errors propagate. For \textbf{MATH} (complex multi-step reasoning), speedup (10.2$\times$) balances between GSM8K's predictability and HumanEval's sensitivity. On \textbf{MBPP} (Python programming), we observe similar patterns to HumanEval, with 11.8$\times$ speedup and high accuracy maintenance (13.4\% vs. 13.6\% Fast-dLLM).

\begin{table}[t]
\centering
\caption{Performance on LLaDA-7B-Instruct (A800).}
\label{tab:llada_results}
\resizebox{\columnwidth}{!}{
\begin{tabular}{llcccc}
\toprule
Benchmark & Method & TPS↑ & Steps↓ & Acc.↑ & Speedup \\
\midrule
\multirow{3}{*}{GSM8K}
& Vanilla & 4.1 & 512.0 & 76.6 & 1.0$\times$ \\
& Fast-dLLM (DC) & 48.8 & 110.3 & 75.6 & 11.9$\times$ \\
& Ours & \textbf{55.7} & \textbf{71.3} & 75.5 & \textbf{13.6$\times$} \\
\midrule
\multirow{3}{*}{HumanEval}
& Vanilla & 13.8 & 512.0 & 43.9 & 1.0$\times$ \\
& Fast-dLLM (DC) & 70.6 & 179.2 & 45.7 & 5.1$\times$ \\
& Ours & \textbf{94.4} & \textbf{106.7} & 46.3 & \textbf{6.8$\times$} \\
\midrule
\multirow{3}{*}{MBPP}
& Vanilla & 4.8 & 512.0 & 14.8 & 1.0$\times$ \\
& Fast-dLLM (DC) & 48.5 & 132.8 & 13.6 & 10.1$\times$ \\
& Ours & \textbf{56.4} & \textbf{81.7} & 13.4 & \textbf{11.8$\times$} \\
\midrule
\multirow{3}{*}{MATH}
& Vanilla & 7.2 & 512.0 & 36.8 & 1.0$\times$ \\
& Fast-dLLM (DC) & 60.7 & 167.1 & 35.7 & 8.4$\times$ \\
& Ours & \textbf{73.5} & \textbf{96.8} & 36.0 & \textbf{10.2$\times$} \\
\bottomrule
\end{tabular}
}
\end{table}

\begin{table}[t]
\centering
\caption{Performance on Dream-7B-Instruct (A800).}
\label{tab:dream_results}
\resizebox{\columnwidth}{!}{
\begin{tabular}{llcccc}
\toprule
Benchmark & Method & TPS↑ & Steps↓ & Acc.↑ & Speedup \\
\midrule
\multirow{3}{*}{GSM8K}
& Vanilla & 9.0 & 512.0 & 76.8 & 1.0$\times$ \\
& Fast-dLLM (DC) & 50.6 & 280.6 & 74.6 & 5.6$\times$ \\
& Ours & \textbf{62.2} & \textbf{164.9} & 73.7 & \textbf{6.9$\times$} \\
\midrule
\multirow{3}{*}{HumanEval}
& Vanilla & 16.9 & 512.0 & 54.2 & 1.0$\times$ \\
& Fast-dLLM (DC) & 52.3 & 295.0 & 51.2 & 3.1$\times$ \\
& Ours & \textbf{78.1} & \textbf{180.7} & 54.3 & \textbf{4.6$\times$} \\
\bottomrule
\end{tabular}
}
\end{table}

\subsection{Ablation: Tree Structure and Inter-Block}
\label{sec:ablation}

Table~\ref{tab:ablation} ablates tree configurations and inter-block speculation on HumanEval. Larger trees (W3D3) provide higher acceptance rates but incur heavier per-step verification cost, resulting in worse latency on A800. Inter-block speculation further reduces steps by $\sim$12\% and increases tokens-per-step to 4.28, demonstrating the value of cross-block lookahead.

\begin{table}[t]
\centering
\caption{Ablation on LLaDA-7B (HumanEval, A800).}
\label{tab:ablation}
\small
\setlength{\tabcolsep}{5pt}
\begin{tabular}{lccc}
\toprule
Configuration & Tokens/Step & Steps & Latency (s) \\
\midrule
Fast-dLLM (DC) & 2.61 & 179.2 & 6.6 \\
\midrule
+ Intra W1D1(1) & 3.51 & 134.5 & 5.6 \\
+ Intra W2D2(3) & 3.83 & 124.0 & \textbf{5.3} \\
+ Intra W3D3(6) & \textbf{3.95} & \textbf{120.1} & 5.5 \\
\midrule
+ Inter W2D2(3) & 4.28 & 106.7 & \textbf{5.0} \\
+ Inter W3D3(6) & \textbf{4.52} & \textbf{105.1} & 5.6 \\
\bottomrule
\end{tabular}
\end{table}

\subsection{Computational Overhead Analysis}
\label{sec:overhead}

Figure~\ref{fig:computehead} analyzes per-step latency as a function of the number of draft blocks. On A800, exploring up to 4 draft blocks keeps overhead acceptable ($<$30\%), aligning with our observed 30-40\% step reduction. On H800, higher compute headroom enables larger structures (W3D3) to achieve better end-to-end latency. This roofline-style analysis confirms that the optimal tree size depends on hardware capability and the step-reduction benefit.

\begin{figure}[t]
\centering
\includegraphics[width=\columnwidth]{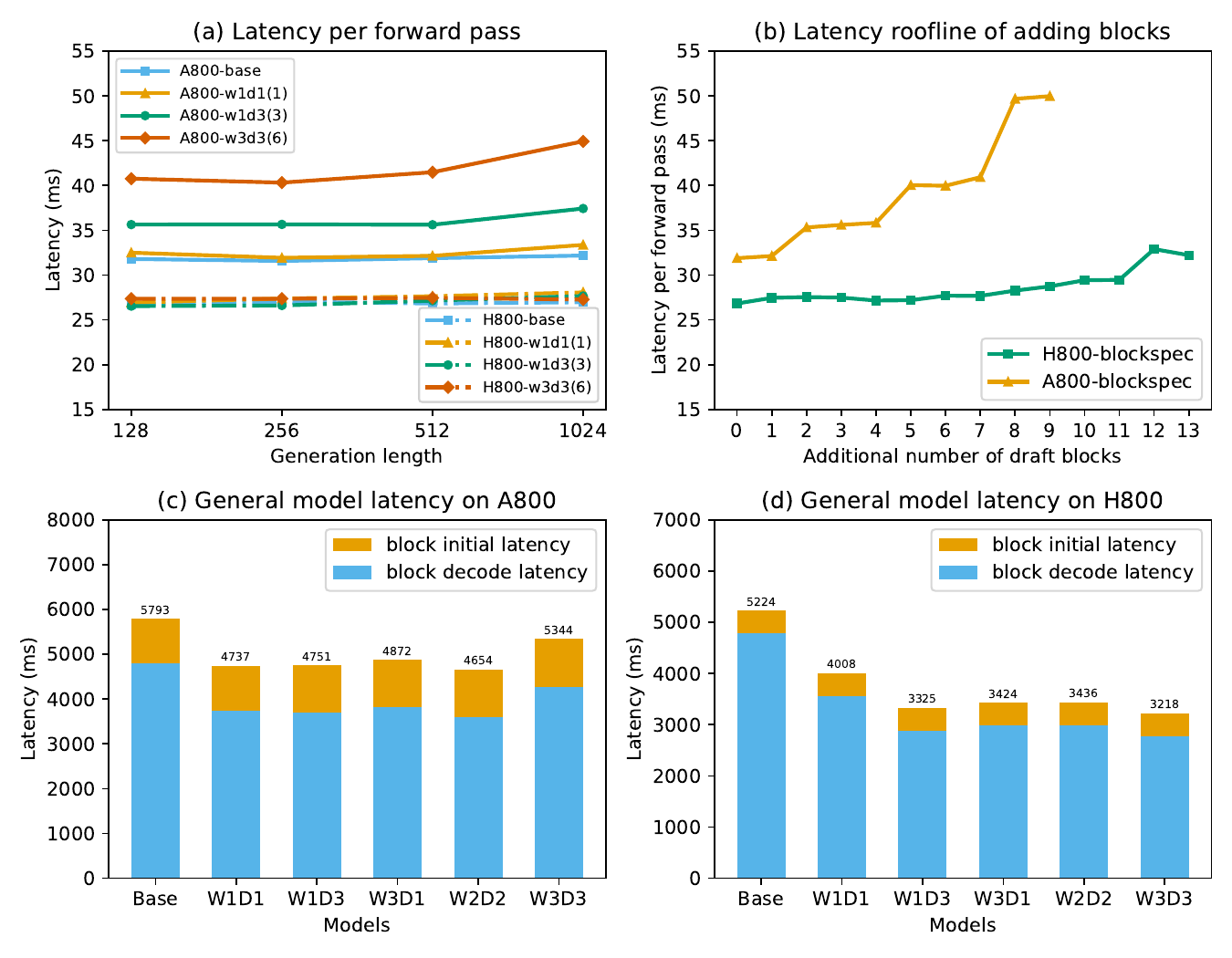}
\caption{Computational overhead analysis. Per-step latency increases with the number of draft blocks, but the reduction in total steps yields net speedup. Dashed line indicates cumulative step reduction. Optimal tree size depends on hardware capability and the overhead-reduction trade-off.}
\label{fig:computehead}
\end{figure}

\textbf{Tree construction overhead.}
To quantify the non-model cost of tree construction, we measure the normalized overhead ratio (as a percentage of total forward pass time) on A100 across tree configurations and sequence lengths (Table~\ref{tab:tree_overhead}). 
Tree construction time was measured using PyTorch CUDA Events and normalized against model forward passes. Model forward latency was profiled with Nsight Systems.
For our default W2D2 configuration, overhead stays below 2\% at sequence length 512 and decreases further at 2048.
Overhead grows with tree size but remains consistently lightweight, confirming that constrained tree construction adds negligible cost beyond the model forward pass.

\begin{table}[t]
\centering
\caption{Tree construction overhead ratio (\%) on A100 across sequence lengths.}
\label{tab:tree_overhead}
\small
\setlength{\tabcolsep}{5pt}
\begin{tabular}{lcccc}
\toprule
& \multicolumn{4}{c}{Sequence Length} \\
\cmidrule(lr){2-5}  
Configuration & 32 & 128 & 512 & 2048 \\
\midrule
w1d1 & 2.52 & 3.50 & 1.27 & 0.86 \\
w2d1 & 7.57 & 6.94 & 1.69 & 1.15 \\
w3d1 & 6.50 & 5.80 & 1.54 & 1.08 \\
w2d2 & 8.25 & 6.72 & 1.86 & 1.29 \\
w2d3 & 7.64 & 4.97 & 1.80 & 1.35 \\
w3d3 & 8.22 & 5.15 & 1.99 & 1.39 \\
w7d1 & 7.87 & 4.83 & 2.04 & 1.48 \\
w8d1 & 6.01 & 4.05 & 1.84 & 1.28 \\
w9d1 & 6.07 & 4.10 & 1.92 & 1.38 \\
\bottomrule
\end{tabular}
\end{table}

\textbf{Attention kernel latency.}
For block-level inference, the chip's roofline characteristics provide redundant space, which allows to add more computational load within the same time cost until reaching the compute bound. 
Table~\ref{tab:attn_latency} reports absolute fused multi-head attention kernel latency (in $\mu$s) on A100 and H100.
Compact configurations (W1D1 through W2D2) incur only marginal overhead relative to the single-block baseline (W1D1: $\sim$60~$\mu$s on A100), while larger trees approach the roofline limit.
The roofline critical points of A100 and H100 are around 200 and 300 concurrent tokens respectively, corresponding to approximately 6 and 9 additional draft blocks with block size 32—well above our default W2D2(3) configuration, leaving ample headroom before compute saturation.
This analysis explains why larger trees such as W3D3 yield diminishing latency returns on A800 but remain beneficial on H800, as shown in Figure~\ref{fig:computehead}.

\begin{table}[t]
\centering
\caption{Fused multi-head attention latency (in $\mu$s) on A100 and H100.}
\label{tab:attn_latency}
\small
\setlength{\tabcolsep}{5pt}
\begin{tabular}{lcc}
\toprule
Configuration & A100 ($\mu$s) & H100 ($\mu$s) \\
\midrule
w1d1 & 60.38 & 37.22 \\
w2d1 & 60.22 & 36.03 \\
w3d1 & 63.14 & 37.76 \\
w2d2 & 63.23 & 37.82 \\
w2d3 & 74.18 & 43.23 \\
w3d3 & 97.95 & 41.95 \\
w7d1 & 99.26 & 44.24 \\
w8d1 & 111.46 & 69.34 \\
w9d1 & 114.50 & 71.30 \\
\bottomrule
\end{tabular}
\end{table}

\section{Discussion and Future Directions}

\textbf{Trajectory drift and exactness.}
Unlike autoregressive speculative decoding, our method introduces small trajectory deviations when increasing parallelism. Theorem~\ref{thm:exactness} establishes that the method is exact under top-1 decoding, but higher parallelism can cause drift. Empirically, accuracy loss remains under 1\%, and failures concentrate on problems requiring precise multi-step reasoning (Appendix~\ref{app:failure}). Developing provably lossless verification for diffusion models remains an open challenge. Potential approaches include: (1) \textit{confidence-aware verification} that adaptively reduces parallelism when trajectory drift risk is high, (2) \textit{rollback mechanisms} that detect and correct divergence mid-generation, and (3) \textit{joint optimization} of the diffusion schedule and speculation strategy to minimize drift while maintaining speedup.

\textbf{Adaptive tree construction.}
The acceptance rate of draft trajectories depends on model confidence, task difficulty, and tree structure. Currently, we rely on heuristic tree designs (e.g., W2D2). Learned acceptance predictors could dynamically adjust tree depth and width based on: (1) historical acceptance rates on similar prompts, (2) confidence distributions at the current step, and (3) task-specific patterns (e.g., code generation benefits from narrower trees). This would enable instance-adaptive speculation that maximizes efficiency per prompt.

\section{Conclusion}

We introduced trajectory-level speculative decoding, the first speculative framework designed specifically for diffusion language models. By constructing tree-structured draft trajectories, performing blockwise parallel verification, and introducing inter-block speculation, our method addresses the fundamental challenges that prevent autoregressive speculative techniques from transferring to dLLMs. Theoretical analysis establishes exactness conditions and identifies trajectory drift as the primary source of deviation. Empirically, we achieve 7-14$\times$ speedup over vanilla dLLMs and 1.3$\times$ over the Fast-dLLM baseline, with minimal accuracy loss.

\section*{Acknowledgements}
We would like to sincerely thank the anonymous reviewers for their valuable comments and constructive suggestions, which have greatly helped us improve the quality and presentation of this paper. We also appreciate our colleagues for their insightful discussions and valuable feedback throughout the research work. We would like to thank Li Auto Inc. for providing financial support and research resources to this work.

\section*{Impact Statement}

This paper presents work whose goal is to advance the field of machine learning by improving inference efficiency for diffusion-based language models. The acceleration techniques we propose could reduce computational costs and energy consumption for deploying these models. There are no specific ethical concerns unique to this work beyond those inherent to large language model deployment in general.

\bibliography{iclr2026_conference}

\begin{thebibliography}{19}
\providecommand{\natexlab}[1]{#1}
\providecommand{\url}[1]{\texttt{#1}}
\expandafter\ifx\csname urlstyle\endcsname\relax
  \providecommand{\doi}[1]{doi: #1}\else
  \providecommand{\doi}{doi: \begingroup \urlstyle{rm}\Url}\fi

\bibitem[Agrawal et~al.(2025)Agrawal, Garrepalli, Goel, Lee, Lott, and
  Porikli]{agrawal2025spiffy}
Agrawal, S., Garrepalli, R., Goel, R., Lee, M., Lott, C., and Porikli, F.
\newblock Spiffy: Multiplying diffusion llm acceleration via lossless
  speculative decoding.
\newblock \emph{arXiv preprint arXiv:2509.18085}, 2025.

\bibitem[Cai et~al.(2024)Cai, Li, Geng, Peng, Lee, Chen, and
  Dao]{pmlr-v235-cai24b}
Cai, T., Li, Y., Geng, Z., Peng, H., Lee, J.~D., Chen, D., and Dao, T.
\newblock Medusa: Simple {LLM} inference acceleration framework with multiple
  decoding heads.
\newblock In Salakhutdinov, R., Kolter, Z., Heller, K., Weller, A., Oliver, N.,
  Scarlett, J., and Berkenkamp, F. (eds.), \emph{Proceedings of the 41st
  International Conference on Machine Learning}, volume 235 of
  \emph{Proceedings of Machine Learning Research}, pp.\  5209--5235. PMLR,
  21--27 Jul 2024.

\bibitem[Chen et~al.(2023)Chen, Borgeaud, Irving, Lespiau, Sifre, and
  Jumper]{chen2023accelerating}
Chen, C., Borgeaud, S., Irving, G., Lespiau, J.-B., Sifre, L., and Jumper, J.
\newblock Accelerating large language model decoding with speculative sampling.
\newblock \emph{arXiv preprint arXiv:2302.01318}, 2023.

\bibitem[Chen et~al.(2025)Chen, Huang, Guo, Wei, He, Zhang, Li, Chen,
  et~al.]{chen2025dpad}
Chen, X., Huang, S., Guo, C., Wei, C., He, Y., Zhang, J., Li, H., Chen, Y.,
  et~al.
\newblock Dpad: Efficient diffusion language models with suffix dropout.
\newblock \emph{arXiv preprint arXiv:2508.14148}, 2025.

\bibitem[De~Bortoli et~al.(2025)De~Bortoli, Galashov, Gretton, and
  Doucet]{de2025accelerated}
De~Bortoli, V., Galashov, A., Gretton, A., and Doucet, A.
\newblock Accelerated diffusion models via speculative sampling.
\newblock \emph{arXiv preprint arXiv:2501.05370}, 2025.

\bibitem[Gao et~al.(2025)Gao, Ji, Wang, Qi, Xu, and Zhang]{gao2025self}
Gao, Y., Ji, Z., Wang, Y., Qi, B., Xu, H., and Zhang, L.
\newblock Self speculative decoding for diffusion large language models.
\newblock \emph{arXiv preprint arXiv:2510.04147}, 2025.

\bibitem[Hu et~al.(2025)Hu, Das, Sadigh, and Anari]{hu2025diffusion}
Hu, H., Das, A., Sadigh, D., and Anari, N.
\newblock Diffusion models are secretly exchangeable: Parallelizing {DDPM}s via
  auto speculation.
\newblock In Singh, A., Fazel, M., Hsu, D., Lacoste-Julien, S., Berkenkamp, F.,
  Maharaj, T., Wagstaff, K., and Zhu, J. (eds.), \emph{Proceedings of the 42nd
  International Conference on Machine Learning}, volume 267 of
  \emph{Proceedings of Machine Learning Research}, pp.\  24270--24289. PMLR,
  13--19 Jul 2025.

\bibitem[Kim et~al.(2025)Kim, Shah, Kontonis, Kakade, and Chen]{kim2025train}
Kim, J., Shah, K., Kontonis, V., Kakade, S.~M., and Chen, S.
\newblock Train for the worst, plan for the best: Understanding token ordering
  in masked diffusions.
\newblock In \emph{Forty-second International Conference on Machine Learning},
  2025.

\bibitem[Leviathan et~al.(2023)Leviathan, Kalman, and
  Matias]{leviathan2023fast}
Leviathan, Y., Kalman, M., and Matias, Y.
\newblock Fast inference from transformers via speculative decoding.
\newblock In \emph{International Conference on Machine Learning}, pp.\
  19274--19286. PMLR, 2023.

\bibitem[Li et~al.(2024)Li, Wei, Zhang, and Zhang]{li2024eagle}
Li, Y., Wei, F., Zhang, C., and Zhang, H.
\newblock {EAGLE}: Speculative sampling requires rethinking feature
  uncertainty.
\newblock In \emph{International Conference on Machine Learning}, 2024.

\bibitem[Miao et~al.(2024)Miao, Oliaro, Zhang, Cheng, Wang, Zhang, Wong, Zhu,
  Yang, Shi, et~al.]{miao2024specinfer}
Miao, X., Oliaro, G., Zhang, Z., Cheng, X., Wang, Z., Zhang, Z., Wong, R.
  Y.~Y., Zhu, A., Yang, L., Shi, X., et~al.
\newblock Specinfer: Accelerating large language model serving with tree-based
  speculative inference and verification.
\newblock In \emph{Proceedings of the 29th ACM International Conference on
  Architectural Support for Programming Languages and Operating Systems, Volume
  3}, pp.\  932--949, 2024.

\bibitem[Nie et~al.(2026)Nie, Zhu, You, Zhang, Ou, Hu, Zhou, Lin, Wen, and
  Li]{nie2025large}
Nie, S., Zhu, F., You, Z., Zhang, X., Ou, J., Hu, J., Zhou, J., Lin, Y., Wen,
  J.-R., and Li, C.
\newblock Large language diffusion models.
\newblock \emph{Advances in Neural Information Processing Systems},
  38:\penalty0 50608--50646, 2026.

\bibitem[Sahoo et~al.(2024)Sahoo, Arriola, Schiff, Gokaslan, Marroquin, Chiu,
  Rush, and Kuleshov]{sahoo2024simple}
Sahoo, S., Arriola, M., Schiff, Y., Gokaslan, A., Marroquin, E., Chiu, J.,
  Rush, A., and Kuleshov, V.
\newblock Simple and effective masked diffusion language models.
\newblock \emph{Advances in Neural Information Processing Systems},
  37:\penalty0 130136--130184, 2024.

\bibitem[Song et~al.(2025)Song, Zhang, Luo, Gao, Xia, Luo, Li, Yang, Yu, Qu,
  et~al.]{song2025seed}
Song, Y., Zhang, Z., Luo, C., Gao, P., Xia, F., Luo, H., Li, Z., Yang, Y., Yu,
  H., Qu, X., et~al.
\newblock Seed diffusion: A large-scale diffusion language model with
  high-speed inference.
\newblock \emph{arXiv preprint arXiv:2508.02193}, 2025.

\bibitem[Wu et~al.(2025)Wu, Zhang, Xue, Liu, Diao, Zhu, Luo, Han, and
  Xie]{wu2025fast}
Wu, C., Zhang, H., Xue, S., Liu, Z., Diao, S., Zhu, L., Luo, P., Han, S., and
  Xie, E.
\newblock Fast-dllm: Training-free acceleration of diffusion llm by enabling kv
  cache and parallel decoding.
\newblock \emph{arXiv preprint arXiv:2505.22618}, 2025.

\bibitem[Xia et~al.(2023)Xia, Ge, Wang, Chen, Wei, and Sui]{xia2022speculative}
Xia, H., Ge, T., Wang, P., Chen, S.-Q., Wei, F., and Sui, Z.
\newblock Speculative decoding: Exploiting speculative execution for
  accelerating seq2seq generation.
\newblock In \emph{The 2023 Conference on Empirical Methods in Natural Language
  Processing}, 2023.

\bibitem[Ye et~al.(2025)Ye, Xie, Zheng, Gao, Wu, Jiang, Li, and
  Kong]{ye2025dream}
Ye, J., Xie, Z., Zheng, L., Gao, J., Wu, Z., Jiang, X., Li, Z., and Kong, L.
\newblock Dream 7b: Diffusion large language models.
\newblock \emph{arXiv preprint arXiv:2508.15487}, 2025.

\bibitem[Zheng et~al.(2024)Zheng, Yuan, Yu, and Kong]{zheng2023reparameterized}
Zheng, L., Yuan, J., Yu, L., and Kong, L.
\newblock A reparameterized discrete diffusion model for text generation.
\newblock In \emph{First Conference on Language Modeling}, 2024.

\bibitem[Zhou et~al.(2023)Zhou, Lu, Mishra, Brahma, Basu, Luan, Zhou, and
  Hou]{zhou2023instruction}
Zhou, J., Lu, T., Mishra, S., Brahma, S., Basu, S., Luan, Y., Zhou, D., and
  Hou, L.
\newblock Instruction-following evaluation for large language models.
\newblock \emph{arXiv preprint arXiv:2311.07911}, 2023.

\end{thebibliography}
\bibliographystyle{icml2026}

\newpage
\onecolumn
\appendix

\section{Detailed Experimental Configurations}
\label{app:exp_details}

All experiments use LLaDA-Instruct-7B and Dream-Instruct-7B with xFormers attention to support non-causal masks. Hardware: 8$\times$ NVIDIA A800 GPUs (CUDA 12.6.3). We fix generation length to 512 tokens and block size to 32 for all main experiments. The dual-cache mechanism follows Fast-dLLM: at each block boundary, we compute full KV states for prefix, suffix, and block root, then cache prefix/suffix KV for the entire block while updating only block-level KV during speculation.

\section{Tree Construction Example}
\label{app:tree_example}

Figure~\ref{fig:tree_example_app} provides a concrete example of tree-based draft trajectory construction. The root block contains high-confidence tokens at positions 0-31. Child blocks explore alternative predictions at low-confidence positions (e.g., positions 12, 18, 24). Each path from root to leaf represents a complete denoising trajectory that can be verified in parallel.

\begin{figure}[h]
\centering
\includegraphics[width=0.8\textwidth]{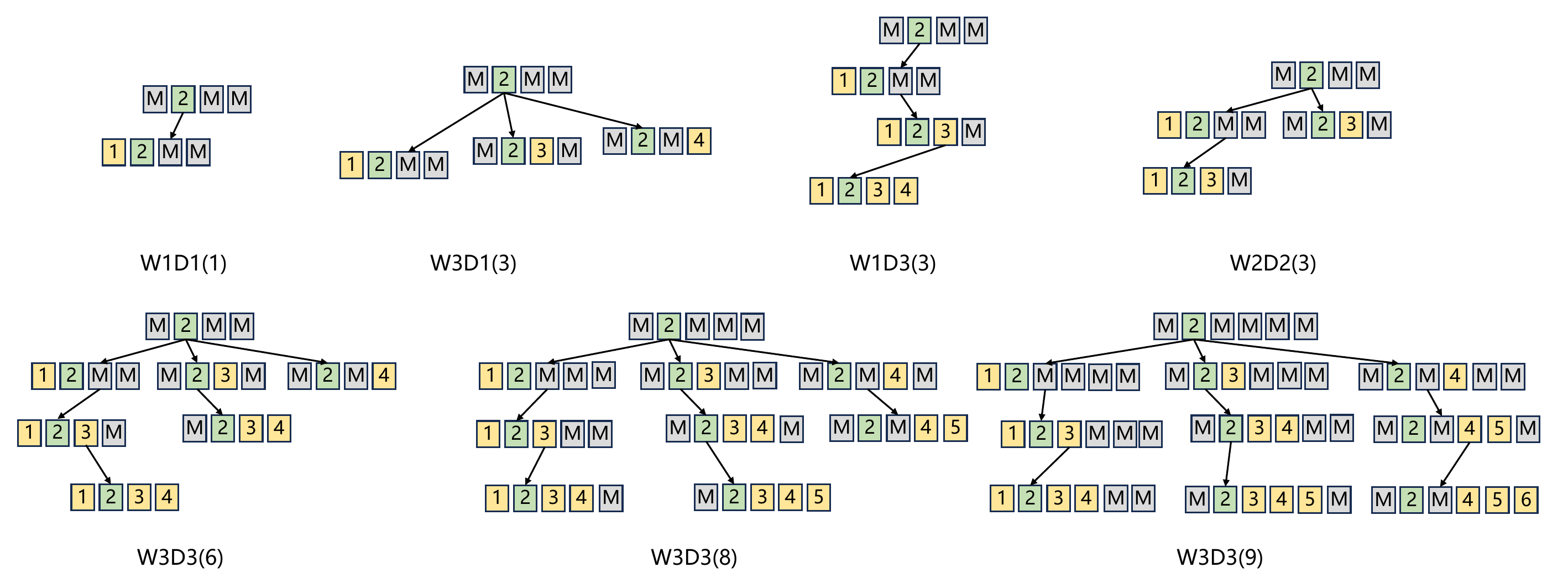}
\caption{Example of tree-based draft trajectory construction. The root block contains high-confidence tokens (positions 0-31). Child blocks explore alternative predictions at low-confidence positions 12, 18, 24. Each path represents a complete denoising trajectory verified in parallel.}
\label{fig:tree_example_app}
\end{figure}

\section{Tree Expansion Strategies}
\label{app:tree_strategies}

We explored multiple tree expansion policies:
(a) Hybrid (ours): top-$k$ at root, top-1 deeper.
(b) Full top-2: expand all layers with top-2.
(c) Full tree: symmetric expansion.
(d) Duplicate-allowed: sibling nodes may share candidates.

Table~\ref{tab:tree_expansion} shows that hybrid achieves the best latency-accuracy trade-off, as larger trees increase overhead without proportional step reduction.

\begin{table}[h]
\centering
\caption{Ablation of tree expansion strategies (HumanEval, A800).}
\label{tab:tree_expansion}
\small
\setlength{\tabcolsep}{4pt}
\begin{tabular}{lcccc}
\toprule
Strategy & Steps & Latency/Step (s) & Total Latency (s) \\
\midrule
Hybrid W2D2(3) & 124.0 & 0.043 & \textbf{5.3} \\
+ top-2 W2D2(4) & 123.9 & 0.045 & 5.6 \\
+ full tree W2D2(4) & 124.9 & 0.045 & 5.6 \\
+ duplicated W2D2(4) & 125.0 & 0.043 & 5.3 \\
\midrule
Hybrid W3D3(6) & 120.1 & 0.046 & \textbf{5.5} \\
+ top-2 W4D3(8) & 119.3 & 0.054 & 6.4 \\
+ full tree W3D3(9) & 120.3 & 0.055 & 6.6 \\
\bottomrule
\end{tabular}
\end{table}

\section{Inter-Block Speculation Details}
\label{app:interblock}

Inter-block speculation triggers when $c_{t+1}^{(1)} > c_t^{(1)}$ or $c_{t+1}^{(1)} > \tau$. We build a compact tree for $B_{t+1}$ (W1D3) and verify both trees jointly. The attention mask preserves bidirectional visibility between the roots of $B_t$ and $B_{t+1}$, while draft tokens within each block remain mutually invisible. Table~\ref{tab:interblock_ablation} shows that adding 1-2 inter-block draft blocks provides optimal speedup.

\begin{table}[h]
\centering
\caption{Inter-block ablation (LLaDA, HumanEval, A800).}
\label{tab:interblock_ablation}
\small
\setlength{\tabcolsep}{4pt}
\begin{tabular}{lccc}
\toprule
Inter-Block Setting & Tokens/Step & Steps & Latency (s) \\
\midrule
No inter-block & 3.83 & 124.0 & 5.3 \\
+ 1 draft block & 4.29 & 109.1 & 5.0 \\
+ 2 draft blocks & \textbf{4.38} & \textbf{106.7} & \textbf{5.0} \\
+ 3 draft blocks & 4.53 & 103.0 & 5.1 \\
\bottomrule
\end{tabular}
\end{table}

\section{Comparison with Concurrent Work}
\label{app:concurrent}

Spiffy~\citep{agrawal2025spiffy} is concurrent work that also explores speculative decoding for diffusion language models through intra-block speculative execution with graph-based calibration. In contrast, our method focuses on trajectory-level speculation and additionally introduces inter-block speculation across consecutive decoding blocks.

Another key difference lies in evaluation protocol. Spiffy mainly reports NFE-based speedup, while we report end-to-end TPS, latency, denoising steps, and task accuracy together. Using the reported numbers from Spiffy and our reported decoding steps on LLaDA-Instruct-7B, we obtain the following reference comparison:

\begin{table}[h]
\centering
\caption{Reference comparison with Spiffy on LLaDA-Instruct-7B.}
\label{tab:spiffy_comparison}
\begin{tabular}{lcc}
\toprule
Benchmark & Spiffy & Ours \\
\midrule
GSM8K & 3.04$\times$ & 7.18$\times$ \\
HumanEval & 2.88$\times$ & 4.87$\times$ \\
MBPP & 2.93$\times$ & 4.42$\times$ \\
MATH & 2.89$\times$ & 5.29$\times$ \\
\bottomrule
\end{tabular}
\end{table}

Our values are computed directly from the denoising steps reported in Table~\ref{tab:spiffy_comparison} using the same normalization convention, while Spiffy's numbers are taken from its reported main results. Since Spiffy is not publicly open-sourced to the best of our knowledge, this comparison should be interpreted as a reported/reference comparison rather than a controlled reproduction study.

\section{Exactness Proof and Empirical Verification}
\label{app:exactness}

\emph{Theorem 1 (full statement).}
Under cached prefix-suffix KV, no inter-block speculation, and deterministic top-1 decoding, our method and Fast-dLLM (Dual Cache) produce identical token trajectories $(i_n, y_n)$ for all $n$.

\emph{Proof.}
By induction on speculative steps. Base case: before step 1, both have empty history. Induction: if trajectories match through step $k-1$, both evaluate the same distribution $p(x_i = t \mid C, \text{history})$ at step $k$ (since $C$ is cached and unchanged). Thus, top-1 selections match, and trajectories remain identical. \qed

\emph{Empirical verification (Table~\ref{tab:exactness_empirical}).}
We run both decoders in top-1 mode on HumanEval and GSM8K; outputs and logits match exactly.

\begin{table}[h]
\centering
\caption{Empirical exactness verification.}
\label{tab:exactness_empirical}
\begin{tabular}{lcc}
\toprule
Dataset & Token Match & Logit Error \\
\midrule
HumanEval & 100\% & 0.00 \\
GSM8K & 100\% & 0.00 \\
\bottomrule
\end{tabular}
\end{table}

\section{Failure Case Analysis}
\label{app:failure}

We analyze failures on MATH (levels 1-5). Overall accuracy drop is $\sim$1\%, concentrated on level-4 problems (Table~\ref{tab:failure_analysis}). Inspection reveals that failures occur when early trajectory drift propagates through multi-step reasoning. Levels 1-3 show no degradation; level-5 (hardest) also shows minimal loss because both methods collapse to low parallelism on very hard problems.

\begin{table}[h]
\centering
\caption{Accuracy by MATH difficulty level.}
\label{tab:failure_analysis}
\begin{tabular}{lcccccc}
\toprule
Model & L1 & L2 & L3 & L4 & L5 & Avg \\
\midrule
Fast-dLLM (DC) & 61.2 & 51.4 & 38.6 & 30.4 & 10.4 & 38.4 \\
Ours & 61.3 & 52.2 & 38.7 & 25.1 & 9.2 & 37.4 \\
\bottomrule
\end{tabular}
\end{table}

\section{Long-Sequence Evaluation}

Table~\ref{tab:long_seq} evaluates generation lengths 128-2048. Our method provides consistent speedup, though relative gains decrease at longer lengths due to early [\textsc{eos}] emission causing both methods to revert to single-step-per-block decoding.

\begin{table}[h]
\centering
\caption{Latency (s) across generation lengths (HumanEval).}
\label{tab:long_seq}
\begin{tabular}{lcc}
\toprule
Length & Fast-dLLM (DC) & Ours \\
\midrule
128 & 1.5 & 1.0 \\
256 & 3.2 & 2.3 \\
512 & 6.6 & 5.0 \\
1024 & 10.2 & 8.7 \\
2048 & 22.4 & 21.2 \\
\bottomrule
\end{tabular}
\end{table}

\section{Analysis of Incremental Speedup over Fast-dLLM}
\label{app:low_conf_speedup}

The speedup over Fast-dLLM~\citep{wu2025fast} is directly governed by the prevalence of low-confidence tokens during inference.
Fast-dLLM (Dual Cache) is already a highly optimized baseline with strong parallel decoding capability; achieving an additional $\approx$1.3$\times$ speedup on top of it is therefore non-trivial.
Concretely, our method delivers $\approx$30--40\% reduction in denoising steps, $\approx$30\% reduction in end-to-end latency, and an increase in tokens-per-step from 2.6 to 4.3 ($\approx$+65\%) on top of this already-optimized system.

\textbf{The role of low-confidence tokens.}
The two approaches address orthogonal bottlenecks: Fast-dLLM reduces per-step computation cost through dual-cache reuse and parallel execution of easy (high-confidence) tokens, while our method reduces the total number of denoising steps by addressing low-confidence degeneration through trajectory-level speculation.
Consequently, our relative gains are strongest precisely when the task induces a high proportion of low-confidence tokens—the regime where threshold-based decoding collapses to near-single-token generation.
Formally, let $\rho$ denote the fraction of decoding steps in which fewer than $k$ tokens exceed the confidence threshold $\tau$.
As $\rho \to 1$ (all steps are low-confidence), the step-count under Fast-dLLM approaches the sequence length, while our method maintains multi-token acceptance via speculative branching.
As $\rho \to 0$ (all steps are high-confidence), both methods converge and the incremental benefit of speculation vanishes.

\textbf{Supplementary evaluation on IFEval.}
To illustrate the above effect in a concrete low-confidence setting, we evaluate on IFEval~\citep{zhou2023instruction}, a benchmark known for long-form instruction-following outputs that exhibit substantially lower model confidence than mathematical or code-generation tasks.
All experiments are conducted on an H800 GPU with confidence threshold $\tau = 0.9$.
As reported in Table~\ref{tab:ifeval_results}, Fast-dLLM generates an average of only 1.3 tokens per step on this task (average generation length 277.2, 210.8 steps), confirming severe low-confidence degeneration.
Our method (W2D2) reduces inference steps to 118.8—a \textbf{44\% reduction}—and achieves nearly $2\times$ improvement in TPS (36.8 $\to$ 69.1), while maintaining comparable accuracy (53.97 $\to$ 54.71).
This represents a substantially larger gain than the 1.2--1.4$\times$ observed on GSM8K and HumanEval, directly corroborating the prediction that our speedup scales with $\rho$.

\begin{table}[h]
\centering
\caption{Performance on IFEval (LLaDA-7B-Instruct, H800, $\tau=0.9$). Fast-dLLM exhibits severe low-confidence degeneration (avg.\ 1.3 tokens/step), enabling our method to achieve nearly $2\times$ speedup.}
\label{tab:ifeval_results}
\begin{tabular}{lcccc}
\toprule
Method & TPS$\uparrow$ & Steps$\downarrow$ & Accuracy & Speedup \\
\midrule
Fast-dLLM (DC) & 36.8 & 210.8 & 53.97$\pm$0.02 & 1.0$\times$ \\
Ours (W2D2) & \textbf{69.1} & \textbf{118.8} & \textbf{54.71$\pm$0.02} & \textbf{1.9$\times$} \\
\bottomrule
\end{tabular}
\end{table}

\textbf{Implementation note.}
Our current implementation operates at the Python execution level, without hardware-specific kernel fusion or CUDA-level parallelization of branch verification.
We reduce decoding steps by $\approx$40\%, yet the per-step overhead from parallel branch verification is not yet fully amortized at the systems level.
Further gains are achievable through more efficient implementations (e.g., fused CUDA kernels for blockwise attention masking), which we view as an engineering opportunity complementary to the algorithmic contributions presented in this work.

\end{document}